\documentclass[letterpaper]{article}
\usepackage[final]{neurips_2020}

\usepackage{amsthm}
\usepackage{amssymb}
\usepackage{amsmath}
\usepackage[utf8]{inputenc}
\usepackage[parfill]{parskip}
\usepackage{graphicx}
\usepackage[export]{adjustbox}
\usepackage{subcaption} 
\usepackage{booktabs}       

\usepackage{comment}
\usepackage{todonotes}

\usepackage[ruled, linesnumbered]{algorithm2e}
\usepackage{hyperref}
\usepackage{float}

\newtheorem{theorem}{Theorem}

\theoremstyle{definition}

\DeclareMathOperator*{\argmax}{argmax}

\title{Novelty Search in Representational Space for Sample Efficient Exploration}
\author{Ruo Yu Tao\textsuperscript{1, 2, *}, Vincent François-Lavet\textsuperscript{1, 2}, Joelle Pineau\textsuperscript{1, 2} \\
\textsuperscript{1} McGill University\\
\textsuperscript{2} Mila, Quebec Artificial Intelligence Institute\\
\textsuperscript{*} \texttt{ruo.tao@mail.mcgill.ca}
}

\begin{document}

\maketitle

\begin{abstract}
We present a new approach for efficient exploration which leverages a low-dimensional encoding of the environment learned with a combination of model-based and model-free objectives. 
Our approach uses intrinsic rewards that are based on the distance of nearest neighbors in the low dimensional representational space to gauge novelty.
We then leverage these intrinsic rewards for sample-efficient exploration with planning routines in representational space for hard exploration tasks with sparse rewards.
One key element of our approach is the use of information theoretic principles to shape our representations in a way so that our novelty reward goes beyond pixel similarity.
We test our approach on a number of maze tasks, as well as a control problem and show that our exploration approach is more sample-efficient compared to strong baselines. 
\end{abstract}

\section{Introduction}

In order to solve a task efficiently in reinforcement learning (RL), one of the main challenges is to gather informative experiences via an efficient exploration of the state space. 
A common approach to exploration is to leverage intrinsic rewards correlated with some metric or score for novelty \citep{schmidhuber2010formal, stadie2015incentivizing, houthooft2016vime}. With intrinsic rewards, an agent can be incentivized to efficiently explore its state space.
A direct approach to calculating these novelty scores is to derive a reward based on the observations, such as a count-based reward \citep{bellemare2016unifying, ostrovski2017count} or a prediction-error based reward \citep{Burda2018Distillation}.
However, an issue occurs when measuring novelty directly from the raw observations, as some information in pixel space (such as randomness or backgrounds) may be irrelevant. In this case, if an agent wants to efficiently explore its state space it should only focus on meaningful and novel information. 

In this work, we propose a method of sample-efficient exploration by leveraging intrinsic rewards in a meaningful latent state space. 
To build a meaningful state abstraction, we view Model-based RL (MBRL) from an information theoretic perspective - we optimize our dynamics learning through the Information Bottleneck \citep{Tishby2000IB} principle. We also combine both model-based and model-free components through a joint representation. This method encodes high-dimensional observations into lower-dimensional representations such that states that are close in \textit{dynamics} are brought close together in representation space \citep{Lavet2018CRAR}. 
We also add additional constraints to ensure that a measure of distance between abstract states is meaningful. We leverage these properties of our representation to formulate a novelty score based on Euclidean distance in low-dimensional representation space and we then use this score to generate intrinsic rewards that we can exploit for efficient exploration.

One important element of our exploration algorithm is that we take a Model Predictive Control (MPC) approach \citep{Garcia1989ModelPC} and perform actions only after our model is sufficiently accurate (and hence ensure an accurate novelty heuristic). Through this training scheme, our agent is also able to learn a meaningful representation of its state space in a sample-efficient manner. The code with all experiments is available \footnote{\url{https://github.com/taodav/nsrs}}.

\section{Problem setting}
An agent interacts with its environment over discrete timesteps, modeled as a Markov Decision Process (MDP), defined by the 6-tuple $(\mathcal{S}, \mathcal{S}_0, \mathcal{A}, \tau, \mathcal{R}, \mathcal{G})$ \citep{Puterman1994MDP}. 
In this setting, 
$\mathcal{S}$ is the state space, 
$\mathcal{S}_0$ is the initial state distribution, 
$\mathcal{A}$ is the discrete action space, 
$\tau: \mathcal{S} \times \mathcal{A} \rightarrow \mathcal{S}$ is the transition function that is assumed deterministic (with the possibility of extension to stochastic environments with generative methods), 
$R: \ \mathcal{S} \times \mathcal{A} \rightarrow \mathcal{R}$ is the reward function ($\mathcal{R}=[-1, 1]$), 
$\mathcal{G}: \mathcal{S} \times \mathcal{A} \rightarrow [0, 1)$ is the per timestep discount factor.
At timestep $t$ in state $s_t \in \mathcal{S}$, the agent chooses an action $a_t \in \mathcal{A}$ based on policy $\pi: \mathcal{S} \times \mathcal{A} \rightarrow [0, 1]$, such that $a_t \sim \pi(s_t, \cdot)$. After taking $a_t$, the agent is in state $s_{t+1} = \tau(s_t, a_t)$ and receives reward $r_t \sim R(s_t, a_t)$ and a discount factor $\gamma_t \sim \mathcal{G}(s_t, a_t)$. Over $n$ environment steps, we define the buffer of previously visited states as $\mathcal{B} = (s_1, \dots, s_{n})$, where $s_i \in \mathcal{S}$ $\forall i \in \mathbb{N}$.
In RL, the usual objective is to maximize the sum of expected future rewards $V_{\pi}(s) = \mathbb{E}_{\pi}\left[r_t + \sum_{i = 1}^{\infty}\left(\prod_{j = 0}^{i - 1}\gamma_{t+j}\right)r_{t + i} | s = s_t \right].$

To learn a policy $\pi$ that maximizes the expected return, an RL agent has to efficiently explore its environment (reach novel states in as few steps as possible). In this paper, we consider tasks with sparse rewards or even no rewards, and are interested in exploration strategies that require as few steps as possible to explore the state space.

\section{Abstract state representations}
\label{sec:abs_state_reps}
We focus on learning a lower-dimensional representation of state when our state (or observations in the partially observable case \citep{Kaelbling1998POMDP}) is high-dimensional \citep{dayan1993improving, tamar2016value, silver2016predictron, oh2017value, de2018integrating, ha2018world,Lavet2018CRAR, Hafner2018PlaNet,gelada2019deepmdp}.

\subsection{Information Bottleneck}
We first motivate our methods for model learning. To do so, we consider the \textit{Information Bottleneck} (IB) \citep{Tishby2000IB} principle.
Let $Z$ denote the original source message space and $\Tilde{Z}$ denote its compressed representation. As opposed to traditional lossless compression where we seek to find corresponding encodings $\Tilde{Z}$ that compresses all aspects of $Z$, in IB we seek to preserve only \textit{relevant} information in $\Tilde{Z}$ with regards to another relevance variable, $Y$. For example when looking to compress speech waveforms ($Z$) if our task at hand is speech recognition, then our relevance variable $Y$ would be a transcript of the speech. Our representation $\Tilde{Z}$ would only need to maximize relevant information about the transcript $Y$ instead of its full form including tone, pitch, background noise etc. We can formulate this objective by minimizing the following functional with respect to $p(\Tilde{z} \ | \ z)$:
\[
    \mathcal{L}(p(\Tilde{z} \ | \ z)) = I[Z ; \Tilde{Z}] - \beta I[\Tilde{Z} ; Y]
\]
where $I[\cdot ; \cdot]$ is the \textit{Mutual Information} (MI) between two random variables. $\beta$ is the Lagrange multiplier for the amount of information our encoding $\Tilde{Z}$ is allowed to quantify about $Y$. This corresponds to a trade-off between minimizing the encoding rate $I[Z ; \Tilde{Z}]$ and maximizing the mutual information between the encoding and our random variable $Y$.

We now apply this principle to representation learning of state in MBRL. If our source message space is our state $S'$ and our encoded message is $X'$, then to distill the most relevant information with regards to the dynamics of our environment one choice of relevance variable is $\{X, A\}$, i.e. our encoded state in the previous timestep together with the presence of an action. This gives us the functional
\begin{equation}
\label{eqn:ib-mbrl}
    \mathcal{L}(p(x' \ | \ s')) = I[S' ; X'] - \beta I[X' ; \{X, A\}].
\end{equation}
In our work, we look to find methods to minimize this functional for an encoding that maximizes the predictive ability of our dynamics model.

We first aim to minimize our encoding rate $I[S'; X']$. Since encoding rate is a measure of the amount of bits transmitted per message $S'$, representation dimension is analogous to number of bits per message. This principle of minimizing encoding rate guides our selection of representation dimension - for every environment, we try to choose the smallest representation dimension possible such that the representation can still encapsulate model dynamics as we understand them. For example, in a simple Gridworld example, we look to only encode agent position in the grid-world. 

Now let us consider the second term in Equation~\ref{eqn:ib-mbrl}. Our goal is to learn an optimally predictive model of our environment. To do so we first consider the MI between the random variable denoting our state representation $X$, in the presence of the random variable representing actions $A$ and the random variable denoting the state representation in the next timestep $X'$ \citep{still2009info}. Note that MI is a metric and is symmetric:
\begin{equation}
    I[\{X, A\} \ ; \ X'] = \mathbb{E}_{p(x', x, a)}\left[ \log \left( \frac{p(x' \ | \ x, a)}{p(x')} \right) \right] = H[X'] - H[X' \ | \ X, A]
\end{equation}
This quantity is a measure of our dynamics model's predictive ability. If we consider the two entropy terms (denoted $H[\cdot]$), we see that $H[X']$ constitutes the entropy of our state representation and $H[X' \ | \ X, A]$ as the entropy of the next state $X'$ given our current state $X$ and an action $A$. Recall that we are trying to minimize $I[X'; S']$ and maximize $I[X'; \{X, A\}]$ with respect to some encoding function $X = e(S)$. In the next section, we describe our approach for this encoding function as well as dynamics learning in MBRL.

\subsection{Encoding and dynamics learning}
For our purposes, we use a neural encoder $\hat{e}: \mathcal{S} \rightarrow \mathcal{X}$ parameterized by $\theta_{\hat{e}}$ to map our high-dimensional state space into lower-dimensional abstract representations, where $\mathcal{X} \subseteq \mathbb{R}^{n_{\mathcal{X}}}$.
The dynamics are learned via the following functions: 
a transition function $\hat{\tau}: \mathcal{X} \times A \rightarrow \mathcal{X}$ parameterized by $\theta_{\hat{\tau}}$, 
a reward function $\hat{r}:  \mathcal{X} \times A \rightarrow [-1, 1]$ parameterized by $\theta_{\hat{r}}$, and 
a per timestep discount factor function $\hat{\gamma}: \mathcal{X} \times A \rightarrow [0, 1)$ parameterized by $\theta_{\hat{\gamma}}$. This discount factor is only learned to predict terminal states, where $\gamma = 0$.

In order to leverage all past experiences, we use an off-policy learning algorithm that samples transition tuples $(s, a, r, \gamma, s')$ from a replay buffer. We first encode our current and next states with our encoder to get $x \leftarrow \hat{e}(s; \theta_{\hat{e}}), \ x' \leftarrow \hat{e}(s'; \theta_{\hat{e}})$.
The Q-function is learned using the DDQN algorithm \citep{Hasselt2015DDQN}, which uses the target:
\[
    Y = r + \gamma Q (\hat{e}(s'; \theta_{\hat{e}^-}), \argmax_{a' \in \mathcal{A}} Q(x', a'; \theta_{Q}); \theta_{Q^-}),
\]
where $\theta_{Q^-}$ and $\theta_{\hat{e}^-}$ are parameters of an earlier buffered Q-function (or our target Q-function) and encoder respectively. The agent then minimizes the following loss:
\[
    L_{Q}(\theta_{Q}) = (Q(x, a; \theta_{Q}) - Y)^2.
\]
We learn the dynamics of our environment through the following losses:
\[
    L_{R}(\theta_{\hat{e}}, \theta_{\hat{r}}) = \left|r - \hat{r}(x, a; \theta_{\hat{r}})\right|^2, \ L_{\mathcal{G}}(\theta_{\hat{e}},\theta_{\hat{\gamma}}) = |\gamma - \hat{\gamma}(x, a; \theta_{\hat{\gamma}})|^2
\]
and our transition loss 
\begin{equation}
\label{eqn:transition-loss}
    L_{\tau}(\theta_{\hat{e}}, \theta_{\hat{\tau}}) = ||[x + \hat{\tau}(x, a; \theta_{\hat{\tau}})] - x'||_2^2.
\end{equation}
Note that our transition function learns the difference (given an action) between previous state $x$ and current state $x'$.
By jointly learning the weights of the encoder and the different components, the abstract representation is shaped in a meaningful way according to the dynamics of the environment. In particular, by minimizing the loss given in Equation~\ref{eqn:transition-loss} with respect to the encoder parameters $\theta_{\hat{e}}$ (or $p(x \ | \ s)$), we minimize our entropy $H[X' | X, A]$.

In order to maximize the entropy of our learnt abstracted state representations $H[X']$, we minimize the expected pairwise Gaussian potential \citep{borodachov2019energy} between states:
\begin{equation}
\label{eqn:maxent}
    L_{d1}(\theta_{\hat{e}}) = \mathbb{E}_{s_1, s_2 \sim p(s)}\left[exp(-C_{d1} ||\hat{e}(s_1; \theta_{\hat{e}}) - \hat{e}(s_2; \theta_{\hat{e}})||_2^2)\right]
\end{equation}
with $C_{d1}$ as a hyperparameter. 
Losses in Equation~\ref{eqn:transition-loss} and Equation~\ref{eqn:maxent} are reminiscent of the model-based losses in \cite{Lavet2018CRAR} and correspond respectively to the \textit{alignment} and \textit{uniformity} contrastive loss formulation in \cite{wang2020understanding}, where alignment ensures that similar states are close together (in encoded representation space) and uniformity ensures that all states are spread uniformly throughout this low-dimensional representation space.

The losses $L_{\tau}(\theta_{\hat{e}})$ and $L_{d1}(\theta_{\hat{e}})$ maximizes the $I[\{X, A\}; X']$ term and selecting smaller dimension for our representation minimizes $I[X', S']$. Put together, our method is trying to minimize $\mathcal{L}(p(x' | s'))$ as per Equation~\ref{eqn:ib-mbrl}.

\subsection{Distance measures in representational space}
For practical purposes, since we are looking to use a distance metric within $\mathcal{X}$ to leverage as a score for novelty, we ensure well-defined distances between states by constraining the $\ell_2$ distance between two consecutive states:
\begin{equation}
\label{eqn:csc}
    L_{csc}(\theta_{\hat{e}}) = max(\lVert \hat{e}(s_1; \theta_e) - \hat{e}(s_2; \theta_e) \rVert_{2} - \omega, 0)
\end{equation}
where $L_{csc}$ is a soft constraint between consecutive states $s_1$ and $s_2$ that tends to enforce two consecutive encoded representations to be at a distance $\omega$ apart. We add $L_{csc}$ to ensure a well-defined $\ell_2$ distance between abstract states for use in our intrinsic reward calculation (a discussion of this loss is provided in Appendix~\ref{appendix:discu_csc}). We discuss how we use $\omega$ to evaluate model accuracy for our MPC updates in Appendix \ref{omega_accuracy}. Finally, we minimize the sum of all the aforementioned losses through gradient descent:
\begin{equation}
    \mathcal{L} = L_{R}(\theta_{\hat{e}},\theta_{\hat{r}}) +L_{\mathcal{G}}(\theta_{\hat{e}},\theta_{\hat{\gamma}}) +L_{\tau}(\theta_{\hat{e}}, \theta_{\hat{\tau}}) + L_{Q}(\theta_{Q}) + L_{d1}(\theta_{\hat{e}}) + L_{csc}(\theta_{\hat{e}}).
\end{equation}
Through these losses, the agent learns a low-dimensional representation of the environment that is meaningful in terms of the $\ell_2$ norm in representation space.
We then employ a planning technique that combines the knowledge of the model and the value function which we use to maximize intrinsic rewards, as detailed in the next section and Section~\ref{combining_mbmf}.

\section{Novelty Search in abstract representational space}

Our approach for exploration uses \textit{intrinsic motivation} \citep{Schmidhuber1991Boredom, Chentanez2005Intrinsic, Achiam2017Surprise} where an agent rewards itself based on the fact that it gathers interesting experiences.
In a large state space setting, states are rarely visited and the count for any state after $n$ steps is almost always 0. While \cite{bellemare2016unifying} solves this issue with density estimation using pseudo-counts directly from the high-dimensional observations, we aim to estimate some function of novelty in our learnt lower-dimensional representation space.

\subsection{Sparsity in representation space as a measure for novelty}
Through the minimization of Equation~\ref{eqn:ib-mbrl}, states that are close together in dynamics are pushed close together in our abstract state space $\mathcal{X}$. Ideally, we want an agent that efficiently explores the \textit{dynamics} of its environment. To do so, we reward our agent for exploring areas in lower-dimensional representation space that are less visited and ideally as far apart from the dynamics that we currently know. 

Given a point $x$ in representation space, we define a reward function that considers the \textit{sparsity} of states around $x$ - we do so with the average distance between $x$ and its $k$-nearest-neighbors in its visitation history buffer $\mathcal{B}$:

\begin{equation}
\label{eqn:novelty}
    \hat{\rho}_{\mathcal{X}}(x) = \frac{1}{k} \sum_{i = 1}^{k}d(x, x_i),
\end{equation}
where $x \ \dot{=} \ \hat{e}(s; \theta_{\hat{e}})$ is a given encoded state, $k \in \mathbb{Z}^+$, $d(\cdot,\cdot)$ is some distance metric in $\mathbb{R}^{n_{\mathcal{X}}}$ and $x_i \ \dot{=} \ \hat{e}(s_i; \theta_{\hat{e}})$, where $s_i \in \mathcal{B}$ for $i = 1 \dots k$ are the $k$ nearest neighbors (by encoding states in $\mathcal{B}$ to representational space) of $x$ according to the distance metric $d(\cdot,\cdot)$. Implicit in this measure is the reliance on the agent's visitation history buffer $\mathcal{B}$.

An important factor in this score is which distance metric to use. With the losses used in Section~\ref{sec:abs_state_reps}, we use $\ell_2$ distance because of the structure imposed on the abstract state space with Equations~\ref{eqn:maxent} and \ref{eqn:csc}.

As we show in Appendix~\ref{appendix:density}, this novelty reward is reminiscent of \textit{recoding probabilities} \citep{bellemare2016unifying, cover2012elements} and is in fact inversely proportional to these probabilities, suggesting that our novelty heuristic estimates visitation count. This is also the same score used to gauge ``sparseness" in behavior space in \cite{Lehman2011Novelty}.

With this reward function, we present the pseudo-code for our exploration algorithm in Algorithm~\ref{algo:novelty_algo}.

\begin{figure}[ht!]
    \centering
    \begin{algorithm}[H]
        \textbf{Initialization:} transition buffer $\mathcal{B}$, agent policy $\pi$\;
        Sample $n_{init}$ initial random transitions, let $t = n_{init}$\;
        \While{$t \leq n_{max}$}{
            \tcp{We update our dynamics model and Q-function every $n_{freq}$ steps}
            \If{$t \mod{n_{freq}} == 0$}{
                \While{$j \leq n_{iters}$ or $L_{\tau} \leq \left( \frac{\omega}{\delta}\right)^2$} {
                    Sample batch of transitions $(s, a, r_{extr}, r_{intr}, \gamma, s') \in \mathcal{B}$\;
                    Train dynamics model with $(s, a, r_{extr}, \gamma, s')$\;
                    Train Q-function with $(s, a, r_{extr} + r_{intr}, \gamma, s')$\;
                }
                $\forall (s, a, r_{extr}, r_{intr}, \gamma, s') \in \mathcal{B}$, set $r_{intr} \leftarrow \hat{\rho}_{\mathcal{X}}(\hat{e}(s'; \theta_{\hat{e}}))$\;
            }
            $a_t \sim \pi(s_t)$\;
            \upshape Take action in environment: $s_{t + 1} \leftarrow \tau(s_t,  a_t)$, $r_{t, extr} \leftarrow R(s_t, a_t)$, $\gamma_{t} \leftarrow \mathcal{G}(s_t, a_t)$\;
            \upshape Calculate intrinsic reward: $r_{t, intr} \leftarrow \hat{\rho}_{\mathcal{X}}(\hat{e}(s_{t + 1}; \theta_{\hat{e}}))$
            
            $\mathcal{B} \leftarrow \mathcal{B} \cup \{ (s_t, a_t, r_{t, extr}, r_{t, intr}, \gamma_{t}, s_{t+1})\}$\;
            
        }
        \caption{The Novelty Search algorithm in abstract representational space. }
        \label{algo:novelty_algo}
    \end{algorithm}

\end{figure}

\subsection{Asymptotic behavior}
This reward function also exhibits favorable asymptotic behavior, as it decreases to 0 as most of the state space is visited. We show this in Theorem~\ref{thm:asymptotics}.

\begin{theorem}
\label{thm:asymptotics}
Assume we have a finite state space $S \subseteq \mathbb{R}^d$, history of states $\mathcal{B} = (s_1, \dots, s_{N})$, encoded state space $\mathcal{X} \subseteq \mathbb{R}^{n_{\mathcal{X}}}$, deterministic mapping $f : \mathbb{R}^d \rightarrow \mathbb{R}^{n_{\mathcal{X}}}$ and a \textit{novelty reward} defined as $\hat{\rho}_{\mathcal{X}}(x)$. With an optimal policy with respect to the rewards of the novelty heuristic, our agent will tend towards states with higher intrinsic rewards. If we assume a communicating MDP setting \citep{Puterman1994MDP}, we have that
\[
\lim\limits_{N \rightarrow \infty} \hat{\rho}_{\mathcal{X}}(f(s)) = 0,\ \forall s \in S.
\]
\end{theorem}
\begin{proof}
We prove this theorem in Appendix~\ref{appendix:limiting-behaviour}.
\end{proof}

\subsection{Combining model-free and model-based components for exploration policies}
\label{combining_mbmf}

Similarly to previous works \citep[e.g.][]{oh2017value, chebotar2017combining}, we use a combination of model-based planning with model-free Q-learning to obtain a good policy. 
We calculate rollout estimates of next states based on our transition model $\hat{\tau}$ and sum up the corresponding rewards, which we denote as $r: \mathcal{X} \times A \rightarrow [0, R_{max}]$ and can be a combination of both intrinsic and extrinsic rewards. We calculate expected returns based on the discounted rewards of our $d$-depth rollouts:
\begin{equation}
    \hat{Q}^d(x, a) = 
    \begin{cases}
    r(x, a) + \hat{\gamma}(x, a; \theta_{\hat{\gamma}}) \times \\
    \ \max\limits_{a' \in \mathcal{A}} \hat{Q}^{d - 1}(\tau(x, a; \theta_{\hat{\tau}}), a'), \ &\text{if } d > 0\\
    Q(x, a; \theta_{Q}), \ &\text{if } d = 0
    \end{cases}
\end{equation}
Note that we simulate only $b$-best options at each expansion step based on $Q (x, a; \theta_Q)$, where $b \leq |\mathcal{A}|$. In this work, we only use full expansions.
The estimated optimal action is given by

$$a^{*} = \underset{a \in \mathcal A}{\operatorname{argmax}} \ \hat{Q}^d(x, a).$$

The actual action chosen at each step follows an $\epsilon$-greedy strategy ($\epsilon \in [0,1]$), where the agent follows the estimated optimal action with probability $1-\epsilon$ and a random action with probability $\epsilon$.

\section{Experiments}
We conduct experiments on environments of varying difficulty. All experiments use a training scheme where we first train parameters to converge on an accurate representation of the already experienced transitions before taking an environment step. 
We optimize the losses (over multiple training iterations) given in Section~\ref{sec:abs_state_reps}. 
We discuss all environment-specific hyperparameters in Appendix~\ref{app:hyperparams}.

\begin{figure}
    \centering
    \captionsetup{width=\linewidth}
    \begin{subfigure}[h]{0.31\linewidth}
    \includegraphics[width=\linewidth]{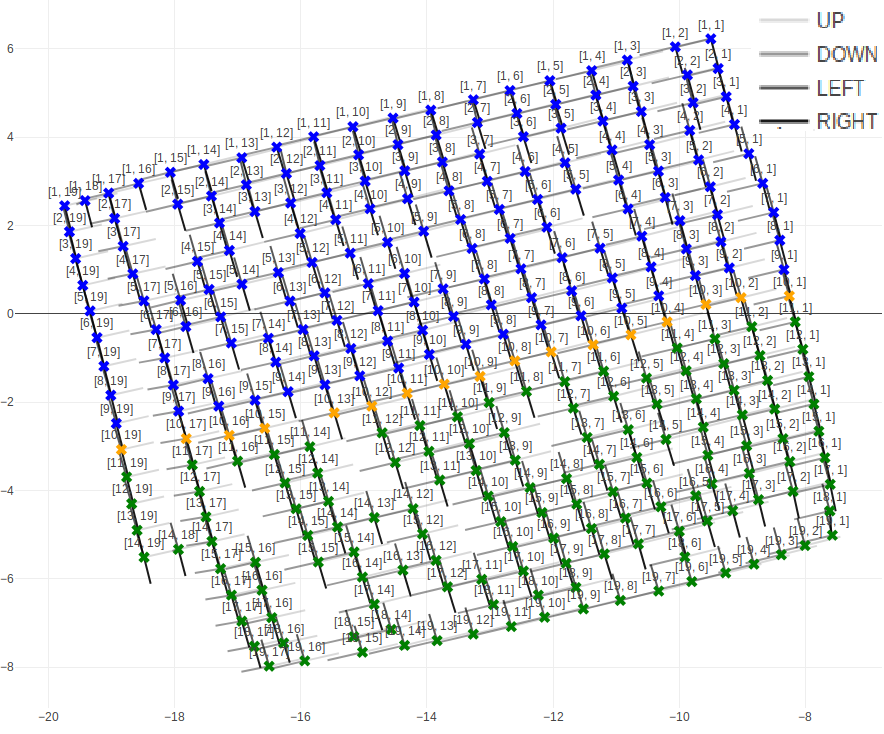}
    \caption{}
    \label{fig:abs_rep_2d}
    \end{subfigure}
    \begin{subfigure}[h]{0.31\linewidth}
    \includegraphics[width=\linewidth]{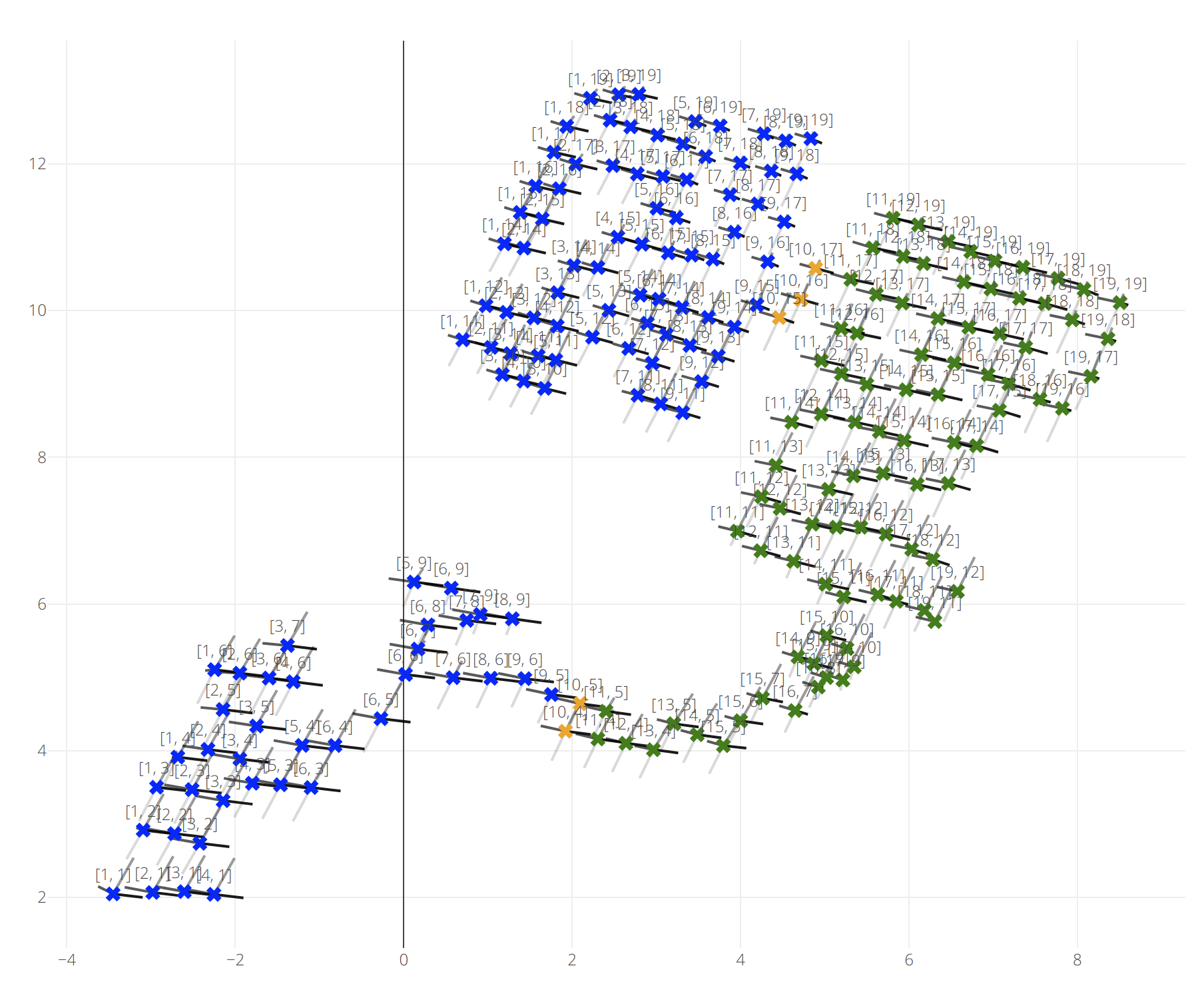}
    \caption{}
    \label{fig:H_maze_representation}
    \end{subfigure}
    \begin{subfigure}[h]{0.31\linewidth}
    \includegraphics[width=\linewidth]{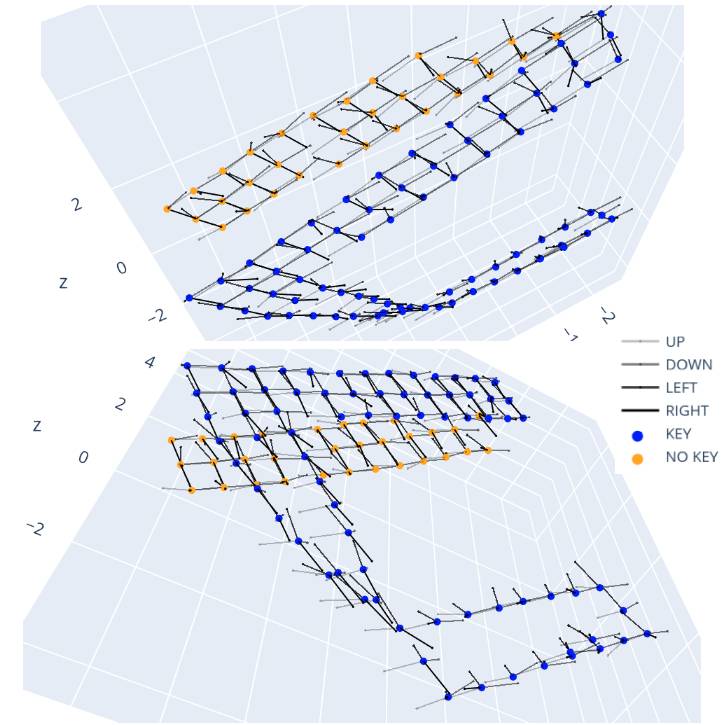}
    \caption{}
    \label{fig:H_maze_representation}
    \end{subfigure}
    \caption{(a), (b): Plotting the full history of learned abstract representations of both open and 4-room labyrinth environments from Figures~\ref{fig:empty_gridworld} and \ref{fig:walled_gridworld} after 500 environment steps. Colors denote which side of the maze the agent was in, grid coordinates and transitions are shown. (c): Two views of the same full history of learned abstract 3-dimensional representation of our multi-step maze after 300 steps. Orange and blue points denote states without and with keys respectively. Our agent is able to disentangle states where the agent has a key and when it doesn't as seen in the distance between the two groups of states. Meaningful information about the agent position is also maintained in the relative positions of states in abstract state space.}
    \label{fig:abstr_rep}
\end{figure}

\subsection{Labyrinth exploration}

We consider two $21 \times 21$ versions of the grid-world environment (Figure~\ref{fig:gridworld_examples} in Appendix). The first is an open labyrinth grid-world, with no walls except for bordering walls. The second is a similar sized grid-world split into four connected rooms.
In these environments the action space $\mathcal{A}$ is the set of four cardinal directions. These environments have no rewards or terminal states and the goal is to explore, agnostic of any task.
We use two metrics to gauge exploration for this environment: the first is the ratio of states visited only once, the second is the proportion of total states visited. 

\subsubsection{Open labyrinth}
In the open labyrinth experiments (Figure~\ref{fig:open_gridworld_results}), we compare a number of variations of our approach with a random baseline and a count-based baseline \citep{bellemare2016unifying} (as we can count states in this tabular setting). 
Variations of the policy include an argmax over state values ($d=0$) and planning depths of $d \in \{1, 5\}$. 
All variations of our method outperform the two baselines in this task, with a slight increase in performance as planning depth $d$ increases.
In the open labyrinth, our agent is able to reach 100\% of possible states (a total of $19 \times 19 = 361$ unique states) in approximately 800 steps, and 80\% of possible states ($\approx 290$ states) in approximately $500$ steps. These counts also include the $n_{init}$ number of random steps taken preceding training. 

Our agent is also able to learn highly interpretable abstract representations in very few environment steps (as shown in Figure~\ref{fig:abs_rep_2d}) as it explores its state space.
In addition, after visiting most unseen states in its environment, our agent tends to uniformly explore its state space due to the nature of our novelty heuristic. A visualisation of this effect is available in Appendix \ref{app:uniform-exploration}.


\subsubsection{4-room labyrinth}
\label{sec:4room}
We now consider the 4-room labyrinth environment, a more challenging version of the open labyrinth environment (Figure~\ref{fig:abs_rep_2d}). 
As before, our encoder $\hat{e}$ is able to take a high-dimensional input and compress it to a low-dimensional representation.
In the case of both labyrinth environments, the representation incorporates knowledge related to the position of the agent in 2-dimensions that we call \textit{primary features}.
In the 4-room labyrinth environment, it also has to learn other information such as agent surroundings (walls, open space) etc., but it does so only via the transition function learned through experience.
We call this extraneous but necessary information \textit{secondary features}.
As most of these secondary features are encoded only in the dynamics model $\hat{\tau}$, our agent has to experience a transition in order to accurately represent both primary and secondary features.

In this environment specifically, our dynamics model might over-generalize for walls between rooms and can sometimes fail at first to try out transitions in the passageways between rooms.
However, because our agent tends to visit uniformly all the states that are reachable within the known rooms,
the $\epsilon$-greedy policy of our approach still ensures that the agent explores passageways efficiently even in the cases where it has over-generalized to the surrounding walls.

We run the same experiments on the 4-room labyrinth domain as we do on the open labyrinth and report results in Figure~\ref{fig:walled_gridworld_results}. In both cases, our method outperforms the two baselines in this domain (random and count-based).

\begin{figure}
    \centering
    \captionsetup{width=\linewidth}
    \begin{subfigure}[ht]{0.49\linewidth}
    \includegraphics[width=\linewidth]{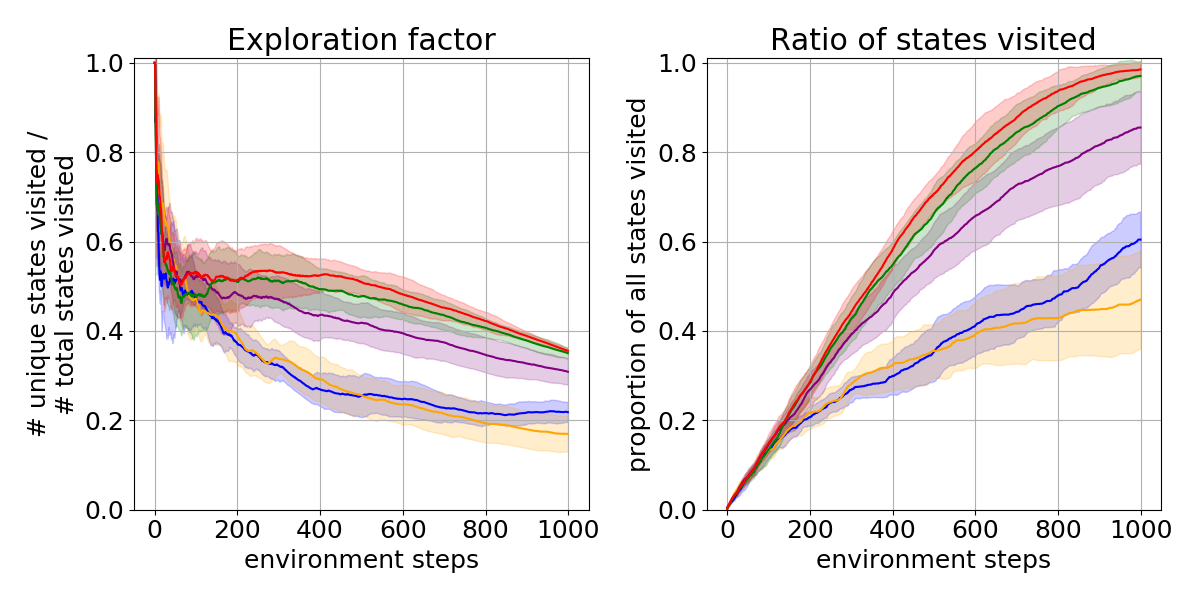}
    \caption{Results for open labyrinth and different variations on policies compared to baselines.}
    \label{fig:open_gridworld_results}
    \end{subfigure}
    \hfill
    \begin{subfigure}[ht]{0.49\linewidth}
    \includegraphics[width=\linewidth]{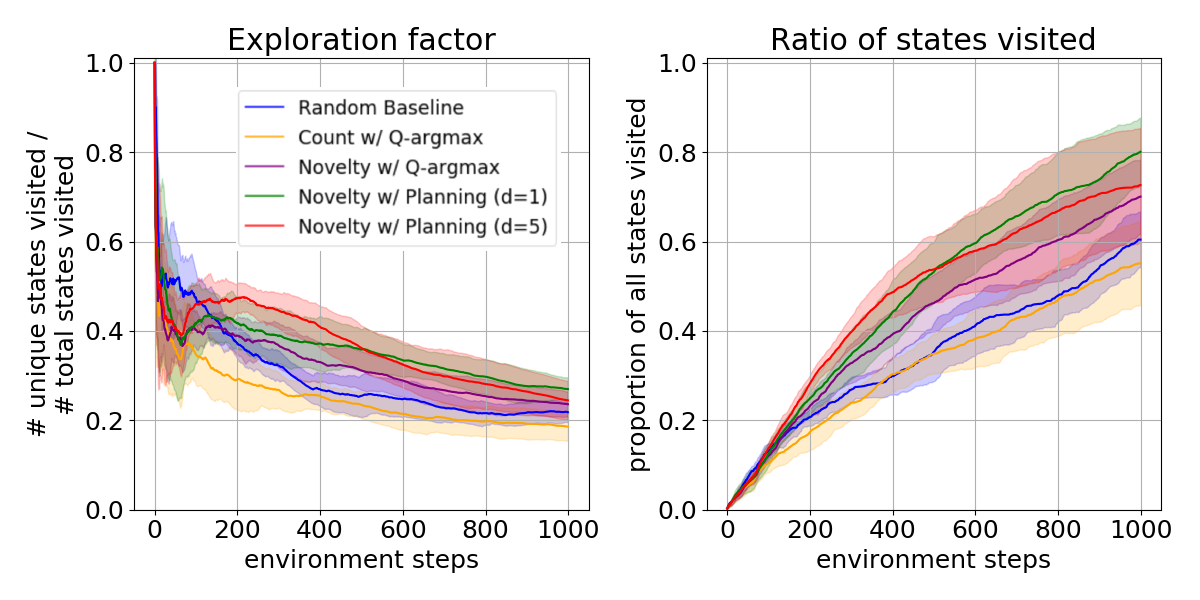}
    \caption{Results for the 4-room labyrinth and different variations on policies compared to baselines.}
    \label{fig:walled_gridworld_results}
    \end{subfigure}
    \caption{Labyrinth results for both open labyrinth and 4-room labyrinth over 10 trials, showing mean and standard deviations.}
    \label{fig:gridworld_results}
\end{figure}

\subsection{Control and sub-goal exploration}
In order to test the efficacy of our method beyond fixed mazes, we conduct experiments on the control-based environment Acrobot \citep{Brockman2016Gym} and a multi-step maze environment. Our method (with planning depth $d = 5$) is compared to strong exploration baselines with different archetypes:
\begin{enumerate}
    \item Prediction error incentivized exploration \citep{Stadie2015TransitionLoss}
    \item Hash count-based exploration \citep{Tang2016Hash}
    \item Random Network Distillation \citep{Osband2017Randomized}
    \item Bootstrap DQN (BDQN, \cite{Osband2016Bootstrap})
\end{enumerate}
In order to maintain consistency in our results, we use the same deep learning architectures throughout. 
Since we experiment in the deterministic setting, we exclude baselines that require some form of stochasticity or density estimation as baselines (for example, \cite{Shyam2018MAX} and \cite{Osband2017Randomized}). 
A specificity of our approach is that we run multiple training iterations in between each environment step for all experiments, which allows the agent to use orders of magnitude less samples as compared to most model-free RL algorithms (all within the same episode).

\subsubsection{Acrobot}
We now test our approach on Acrobot \citep{Brockman2016Gym}, which has a continuous state space unlike the labyrinth environment.
We specifically choose this control task because the nature of this environment makes exploration inherently difficult. The agent only has control of the actuator for the inner joint and has to transfer enough energy into the second joint in order to swing it to its goal state. 
We modify this environment so that each episode is at most $3000$ environment steps. While this environment does admit an extrinsic reward, we ignore these rewards entirely.
To measure the performance of our exploration approach, we measure the average number of steps per episode that the agent takes to move its second joint above a given line as per Figure~\ref{fig:acrobot_examples}.

To demonstrate the ability of our method to learn a low dimensional abstract representation from pixel inputs, we use 4 consecutive pixel frames as input instead of the $6$-dimensional full state vector. 
We use a $4$-dimensional abstract representation of our state and results from experiments are shown in Table~\ref{tab:acrobot_hmaze_results}.
Our method reaches the goal state more efficiently than the baselines. 

\begin{figure}[ht]
    \captionsetup{width=\linewidth}
    \begin{subfigure}{0.49\linewidth}
        \centering
        \begin{subfigure}[ht]{0.45\linewidth}
        \includegraphics[width=\linewidth, frame]{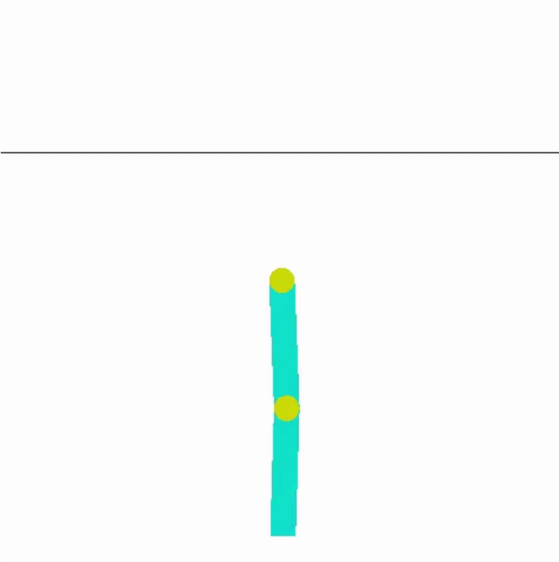}
        \end{subfigure}
        \ \
        \begin{subfigure}[ht]{0.45\linewidth}
        \includegraphics[width=\linewidth, frame]{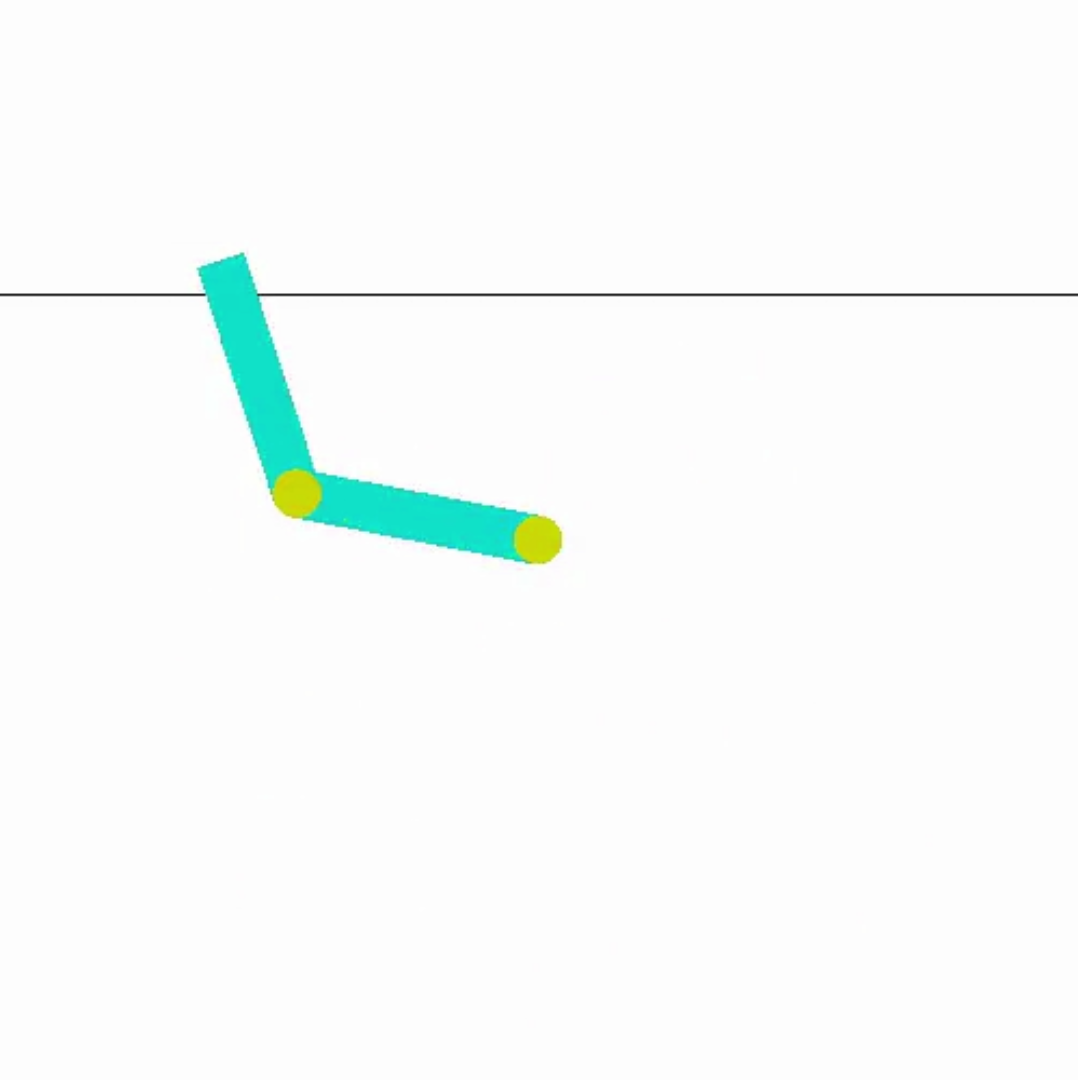}
        \end{subfigure}%
        \caption{\textit{Left}: Acrobot start state. \textit{right}: Acrobot end state}
        \label{fig:acrobot_examples}

    \end{subfigure}
    \begin{subfigure}{0.49\linewidth}
        \centering
        \begin{subfigure}[h]{0.45\linewidth}
        \includegraphics[width=\linewidth, frame]{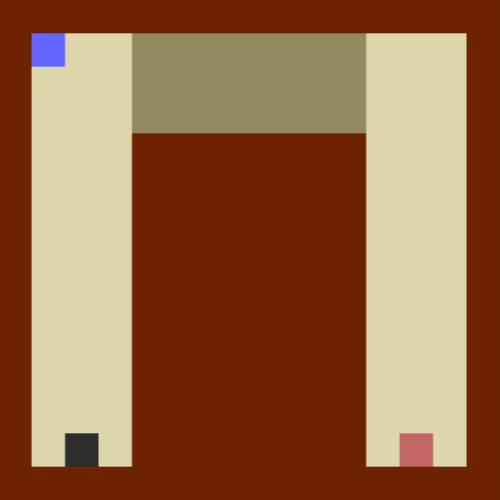}
        \end{subfigure}
        \ \
        \begin{subfigure}[h]{0.45\linewidth}
        \includegraphics[width=\linewidth, frame]{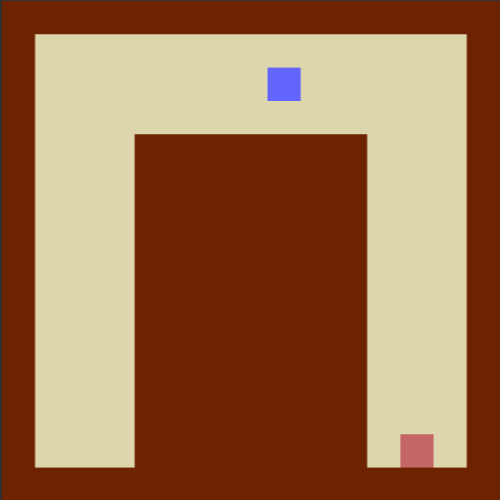}
        \end{subfigure}%
        \caption{\textit{Left}: Start of our multi-step maze. \textit{right}: After the agent has collected the key.}
        \label{fig:multi_step}

    \end{subfigure}

    \caption{Illustrations of the Acrobot and multi-step goal maze environments.
    \textit{a) Left:} The Acrobot environment in one configuration of its start state.
    \textit{a) Right:} One configuration of the ending state of the Acrobot environment. The environment finishes when the second arm passes the solid black line.
    \textit{b) Left:} The passageway to the west portion of the environment is blocked before the key (black) is collected. \textit{b) Right:} The passageway is traversable after collecting the key, and the reward (red) is then available. The environment terminates after collecting the reward.}
    \label{fig:environments}
\end{figure}

\begin{table}
    \centering
    \captionsetup{width=\linewidth}

    \begin{tabular}{|r|c|c|c|c|c|c|c|c|c|}
        \toprule
        & \multicolumn{3}{|c|}{Acrobot} 
        & \multicolumn{3}{c|}{Multi-step Maze}
        & \multicolumn{3}{c|}{Norm. \& Combined}\\
        \midrule
        Reward & \text{Avg} & \text{StdErr} & \text{p-value} & \text{Avg} & \text{StdErr} & \text{p-value} & \text{Avg} & \text{StdErr} & \text{p-value}\\
        \midrule
        \text{Random} & 1713.3 & 316.25 & 0.0077 & 1863.3 & 308.35 & 0.0025 & 3.26 & 0.41 & $3.1e^{-5}$\\
        \text{Pred} & 932.8 & 141.54 & 0.050 & 1018.0 & 79.31 & $4.0e^{-4}$ & 1.78 & 0.15 & $1.3e^{-4}$\\
        \text{Count} & 1007.0 & 174.81 & 0.050 & 658.8 & 71.73  & 0.23 & 1.50 & 0.18 & 0.019\\
        \text{RND} & 953.8 & 85.98 & 0.0042 & 938.4 & 135.88 & 0.024 & 1.72 & 0.15 & $3.5e^{-4}$\\
        \text{BDQN} & 592.5 & 43.65 & 0.85 & 1669.1 & 291.26 & 0.0046 & 2.11 & 0.37 & 0.0099\\
        \text{Novelty} & \textbf{576.0} & 66.13 & - & \textbf{524.6} & 73.24 & -  & \textbf{1.00} & 0.090 & -\\
        \bottomrule
    \end{tabular}
    \medskip
    \caption{Number of environment steps necessary to reach the goal state in the Acrobot and the multi-step maze environments (lower is better). Results are averaged over 5 trials for both experiments. Best results are in bold. We provide p-values indicative of the null hypothesis
    $H_0 : \Delta \mu = \mu_1 - \mu_2 = 0$, calculated using Welch’s t-test, all as per \citep{Colas2019Hitchiker}. In this case, we do a pair-wise comparison between the central tendencies of our algorithm (Novelty) and our baselines. Normalized and combined results are also shown - results here are first normalized with respect to the average number of steps taken for our algorithm and then combined on both environments.}
    \label{tab:acrobot_hmaze_results}
\end{table}
\subsubsection{Multi-step goal maze}
\label{section:multi_step_maze}
We also test our method on a more complex maze with the sub-task of picking up a key that opens the door to an area with a reward. We build our environment with the Pycolab game engine \citep{Stepleton2017Pycolab}. The environment can be seen in Figure \ref{fig:multi_step}, where the input to our agent is a top-down view of the environment. While this environment does admit an extrinsic reward (1 for picking up the key, 10 for reaching the final state), we ignore these rewards and only focus on intrinsic rewards.

In our experiments, we show that our agent is able to learn an interpretable representation of the environment in a sample-efficient manner. Figure~\ref{fig:H_maze_representation} shows an example of learnt representations in this domain after reaching the goal - we observe that positions in the maze correspond to a nearly identical structure in the lower-dimensional representation. Our representation also nicely captures internal state information (whether the key has been picked up) by separating the two sets of states (states when the key has been picked up and states when the key has not been picked up). Similar positions in both sets of states are also mapped closely together in lower-dimensional space (ie. (1, 1, \textit{with key}) is close in $\ell_2$ to (1, 1, \textit{without key})), suggesting good generalization between similar states.

\section{Related work}
The proposed exploration strategy falls under the category of directed exploration \citep{thrun1992efficient} that makes use of the past interactions with the environment to guide the discovery of new states.
This work is inspired by the Novelty Search algorithm \citep{Lehman2011Novelty} that uses a nearest-neighbor scoring approach to gauge novelty in policy space. Our approach leverages this scoring to traverse dynamics space, which we motivate theoretically.
Exploration strategies have been investigated with both model-free and model-based approaches. In \cite{bellemare2016unifying} and \cite{ostrovski2017count}, a model-free algorithm provides the notion of novelty through a pseudo-count from an arbitrary density model that provides an estimate of how many times an action has been taken in similar states. Recently, \cite{Taiga2020BonusALE} do a thorough comparison between bonus-based exploration methods in model-free RL and show that architectural changes may be more important to agent performance (based on extrinsic rewards) as opposed to differing exploration strategies. 

Several exploration strategies have also used a model of the environment along with planning.
\cite{Hester2012IntrinsicallyMM} employ a two-part strategy to calculate intrinsic rewards, combining model uncertainty (from a random-forest based model) and a novelty reward based on $L_1$ distance in feature space.
A strategy investigated in \citet{salge2014changing,mohamed2015variational,gregor2016variational,chiappa2017recurrent} is to have the agent choose a sequence of actions by planning that leads to a representation of state as different as possible to the current state.
In \citet{pathak2017curiosity,haber2018learning}, the agent optimizes both a model of its environment and a separate model that predicts the error/uncertainty of its own model.
\cite{Burda2018Curosity} similarly uses an intrinsic reward based on the uncertainty of its dynamics model.
In \cite{Shyam2018MAX}, forward models of the environment are used to measure novelty derived from disagreement between future states.
\cite{Still2012InfoCuriosity} take an information theoretic approach to exploration, that chooses a policy which maximizes the predictive power of the agent's own behavior and environment rewards.
In \cite{badia2020up}, an intrinsic reward from the k-NN over the agent's experience is also employed for exploration. They instead employ a self-supervised inverse dynamics model to learn the embeddings as opposed to our approach.
Beyond improved efficiency in exploration, the interpretability of our approach could also lead to human-in-the-loop techniques \citep{mandel2017add,abel2017agent} for exploration, with the possibility for the agent to better utilize feedback from interpretability of the agent in representation space.


\section{Discussion}
In this paper, we formulate the task of dynamics learning in MBRL through the \textit{Information Bottleneck} principle. We present methods to optimize the IB equation through low-dimensional abstract representations of state. We further develop a novelty score based on these learnt representations that we leverage as an intrinsic reward that enables efficient exploration. By using this novelty score with a combination of model-based and model-free approaches for planning, we show more efficient exploration across multiple environments with our learnt representations and novelty rewards.

As with most methods, our approach also has limitations.
One limitation we may have is the scalability of non-parametric methods such as k-NN density estimation since our method scales linearly with the number of environment steps. A possible solution to this problem would be to use some sampling scheme to sample a fixed number of observations for calculation of our novelty heuristic. 
Another issue that has arisen from using very low-dimensional space to represent state is generalization. In some cases, the model can over-generalize with the consequence that the low-dimensional representation loses information that is crucial for the exploration of the entire state space. An interesting direction for future work would be to find ways of incorporating secondary features such as those mentioned in Section~\ref{sec:4room}. An interesting possibility would be to use a similar IB method, but using a full history of states as the conditioning variable. Beyond these points, we discuss limitations and potential improvements to this work in Appendix~\ref{appendix:improvements}.

Finally, we show preliminary results of our method on a more complex task - \textit{Montezuma's Revenge} - in Appendix~\ref{app:montezuma}. With the theory and methods developed in this paper, we hope to see future work done on larger tasks with more complex environment dynamics.

\newpage

\section*{Broader Impact}
Algorithms for exploring an environment are a central piece of learning efficient policies for unknown sequential decision-making tasks. 
In this section, we discuss the wider impacts of our research both in the Machine Learning (ML) field and beyond.

We first consider the benefits and risks of our method on ML applications. 
Efficient exploration in unknown environments has the possibility to improve methods for tasks that require accurate knowledge of its environment. By exploring states that are more novel, agents have a more robust dataset. For control tasks, our method improves the sample efficiency of its learning by finding more novel states in terms of dynamics for use in training.
Our learnt low-dimensional representation also helps the interpretability of our decision making agents (as seen in Figure~\ref{fig:abstr_rep}). More interpretable agents have potential benefits for many areas of ML, including allowing human understandability and intervention in human-in-the-loop approaches.

With such applications in mind, we consider societal impacts of our method, along with potential future work that could be done to improve these societal impacts. One specific instance of how efficient exploration and environment modeling might help is in disaster relief settings. With the incipience of robotic systems for disaster area exploration, autonomous agents need to efficiently explore their unknown surroundings. Further research into scaling these MBRL approaches could allow for these robotic agents to find points of interest (survivors, etc.) efficiently.

One potential risk of our application is safe exploration. Our method finds and learns from states that are novel in terms of its dynamics. Without safety mechanisms, our agent could view potentially harmful scenarios as novel due to the rarity of such a situation. For example, a car crash might be seen as a highly novel state. To mitigate this safety concern we look to literature on Safety in RL \citep{Garcia2015Safety}. In particular, developing a risk metric based on the interpretability of our approach may be an area of research worth developing.

\section*{Acknowledgements}
We would like to thank Emmanuel Bengio for the helpful discussions and feedback on early drafts of this work. We would also like to thank all the reviewers for their constructive and helpful comments.

\newpage
\bibliography{bib}
\bibliographystyle{apalike}

\clearpage

\appendix
\section{Using $\omega$ to gauge model accuracy}
\label{omega_accuracy}
The hyperparameter $\omega$ can be used to estimate the accuracy of our transition loss, hence of our novelty estimates. In order to gauge when our representation is accurate enough to use our novelty heuristic, we use a function of this hyperparameter and transition loss to set a cutoff point for accuracy to know when to take the next environment step.
If $\omega$ is the minimum distance between successive states, then when $L_{\tau} \leq \left(\frac{\omega}{\delta}\right)^2$, the transitions are on average within a ball of radius $\frac{\omega}{\delta}$ of the target state. Here $\delta > 1$ is a hyperparameter that we call the slack ratio. Before taking a new step in the environment, we keep training all the parameters with all these losses until this threshold is reached and our novelty heuristic becomes useful. The abstract representation dimensionality is also another hyperparameter that requires tuning, depending on the complexity of the environment.
Details on the slack ratios, abstract representation dimensionality and other hyperparameters are given in Appendix \ref{app:hyperparams}.

\section{Discussion on the distance between successive encoded states}
\label{appendix:discu_csc}
As for our soft constraints on representation magnitude, we use a local constraint instead of a global constraint on magnitude such that it is more suited for our novelty metric. If we are to calculate some form of intrinsic reward based on distance between neighboring states, then this distance needs to be non-zero and ideally consistent as the number of unique states in our history increases. In the global constraint case, if the intrinsic rewards decrease with an increase in number of states in the agent's history, then the agent will fail to be motivated to explore further. 
Even though the entropy maximization losses ensures the maximization of distances between random states, if we have $|\mathcal{B}|$ number of states in the history of the agent, then a global constraint on representation magnitude might lead to
\begin{equation}
    \lim_{|\mathcal{B}| \rightarrow \infty} \mathbb{E}_{(s, s') \sim (\mathcal{B}, \mathcal{B})}[\lVert s - s' \lVert_2] = 0.
\end{equation}
We also test the necessity of this loss in Appendix~\ref{appendix:csc_ablation} and see that without this loss, we incur a high variance in exploration performance.

\section{Motivation for $\ell_2$ distance}
\label{appendix:l2-motivation}
We consider the consecutive distance loss $L_{csc}$. Minimization of this loss ensures that the distance between two consecutive states is $\leq \omega$. This along with our alignment and uniformity losses, $L_{\hat{\tau}}$ and $L_{d1}$ ensures that temporally close states are close in representational space and states are uniformly spread throughout this space. This implies that the minima of $L_{csc}$ between two consecutive states $s$ and $s'$ will occur when:
\begin{align*}
L^*_{consec} &= \min_{\theta_{\hat{e}}} L_{csc}(\theta_{\hat{e}}) \\
&= \min_{\theta_{\hat{e}}} max(\lVert \hat{e}(s; \theta_e) - \hat{e}(s'; \theta_e) \rVert_{2} - \omega, 0)\\
&= \min_{\theta_{\hat{e}}} \left[ \lVert \hat{e}(s; \theta_e) - \hat{e}(s'; \theta_e) \rVert_{2} - \omega \right]
\end{align*}

The minimum of this loss is obtained when the $\ell_2$ distance between $s$ and $s'$ is $\omega$. When this loss is minimized, $\ell_2$ distance is well-defined in our representation space, which implies that our novelty heuristic will also be well-defined. These losses shape abstract state space so that $\ell_2$ norm as a distance measure encodes a notion of closeness in state space that we leverage in our novelty heuristic.

\section{Novelty heuristic as an inverse recoding probability score}
\label{appendix:density}
Consider $P(X_{n + 1} = x \ | \ X_{1:n} = x_{1:n})$, the \textit{recoding probability} of state $x$ at timestep $n + 1$. We try to show that our novelty heuristic of a state is inversely proportional to its recoding probability, or that:
\[
\rho_{\mathcal{X}}(x) = \frac{c}{P(X_{n + 1} = x \ | \ X_{1:n} = x_{1:n})}
\]
where $c$ is some constant.
If we were to try and estimate our recoding probability first using a non-parametric method then using its inverse, we might consider the K-nearest neighbor approach \citep{loftsgaarden1965nndensity}:
\[
P(X_{n + 1} = x \ | \ X_{1:n} = x_{1:n}) \approx \frac{k}{n V_{x, x_k}}
\]
where $k < n$ is an integer and $V_{x, x_k}$ is the volume around $x$ of radius $d(x, x_k)$, where $x_k \in \mathcal{X}$ is the $k$th nearest neighbor to $x$. The issue with this approach is that our score is only dependent on it's $k$th nearest neighbor (as this score only depends on $V_{x, x_k}$), and doesn't take into account the other $k - 1$ nearest neighbors. In practice, we instead try to find something proportional to the inverse directly: we average each of the 1-nearest neighbor density inverses over the k-nearest neighbors:
\begin{align*}
    \rho_{\mathcal{X}}(x) &= \frac{c}{P(X_{n + 1} = x \ | \ X_{1:n} = x_{1:n})}\\
    &\approx nV_{x, x_1}\\
    &\approx \frac{n}{k}\sum_{i = 1}^{k}V_{x, x_i}
\end{align*}
Since we're only worried about proportionality, we can remove the constant $n$ and replace our volume of radius between two points with a distance metric $d$:
\[
\rho_{\mathcal{X}}(x) \propto \hat{\rho}_{\mathcal{X}}(x) = \frac{1}{k} \sum_{i = 1}^{k}d(x, x_i).
\]
Which is our novelty heuristic.

\section{Limiting behavior for novelty heuristic}
\label{appendix:limiting-behaviour}

\begin{proof}[Proof \ (Theorem~\ref{thm:asymptotics})]
Let $n_s$ be the visitation count for a state $s \in S$. We assume that our agent's policy will tend towards states with higher rewards. Given the episodic nature of MDPs, we have that over multiple episodes all states communicate. Since our state space is finite, we have that
\[
\lim\limits_{n \rightarrow \infty}n_s = \infty, \ \forall s \in S.
\]
which means that $\exists n$ such that $k < n_s$ as $n \rightarrow \infty$, and implies that the $k$ nearest neighbors of $s$ are indiscernible from $s$. Since $f$ is a deterministic function, $x_i = x$ for all $i$. We also assume that our agent's policy will tend towards states with higher rewards. As $x$ and $x_i$ are indiscernible and $dist$ is a properly defined metric, $dist(x, x_i) = 0$ for all $i$, and we have
\begin{align}
    \lim\limits_{n \rightarrow \infty} \hat{\rho}(x) &= \frac{1}{k}\sum_{i = 1}^{k} dist(x, x_i)\\
    &= 0.
\end{align}
\end{proof}

\section{Ablation study}
Here we perform ablation studies to test the affects of our losses and hyperparameters.

\subsection{Consecutive distance loss}
\label{appendix:csc_ablation}
To further justify our use of the $L_{csc}$ loss, we observe the results of running the same trials in the simple maze environment (with no walls) with no $L_{csc}$ loss in Figure~\ref{fig:no_csc_no_walls_ablation}.
\begin{figure}[ht!]
    \centering
    \includegraphics[width=\linewidth]{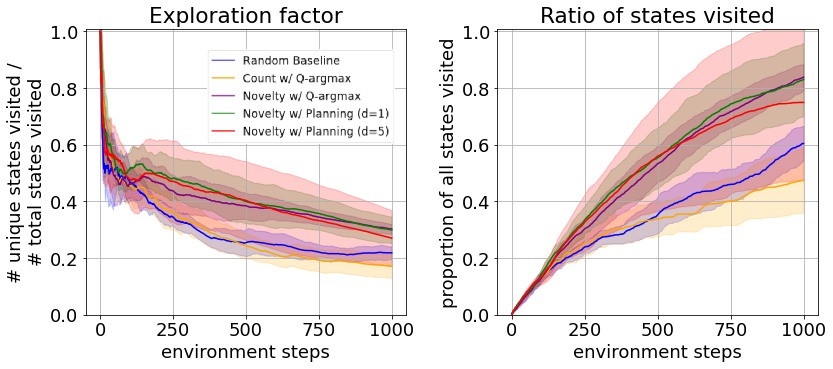}
    \caption{Simple maze (with no walls) experiment with no $L_{csc}$ loss.}
    \label{fig:no_csc_no_walls_ablation}
\end{figure}
As we increase the number of forward prediction steps the exploration is less effective and variance of our results increases. Without the relative distance constraints of our representation, we observe an increase of forward prediction errors, which is the likely cause of the decrease in performance. These forward prediction errors are further compounded as we increase the number of forward prediction steps (as can be seen when comparing the standard error between $d = 0, 1, 5$).
\newpage
\subsection{Pure model-based/model-free}
We test the importance of using a combination of both model-based and model-free components on the multi-step maze environment introduced in Section~\ref{section:multi_step_maze}. We train with the same hyperparameters but in the model free ($d = 0$) and model-based ($d = 5$, no Q-value tails) settings. We show results in Table~\ref{appendix:mbmf-ablation}.
\begin{table}
    \centering
    \captionsetup{width=\linewidth}
    \begin{tabular}{|r|c|c|c|}
        \toprule
        Ablation & Avg ($\mu$) & StdErr & p-value \\
        \midrule
        MF & 758.60 & 169.08 & 0.25\\
        MB & 584.10 & 64.52 & 0.57\\
        Full & 524.60 & 73.24 & - \\
        \bottomrule
    \end{tabular}
    \medskip
    \caption{A further ablation study on the multi-step maze environment. The MF (model-free) ablation does not employ any forward intrinsic reward planning ($d = 0$), while the MB (model-based) ablation only uses forward intrinsic reward planning without using or learning Q-values.}
    \label{appendix:mbmf-ablation}
    
\end{table}

\section{Montezuma's Revenge visualizations}
\label{app:montezuma}

\begin{figure}[ht!]
    \centering
    \captionsetup{width=\linewidth}

    \begin{subfigure}[h]{0.53\linewidth}
    \includegraphics[width=\linewidth]{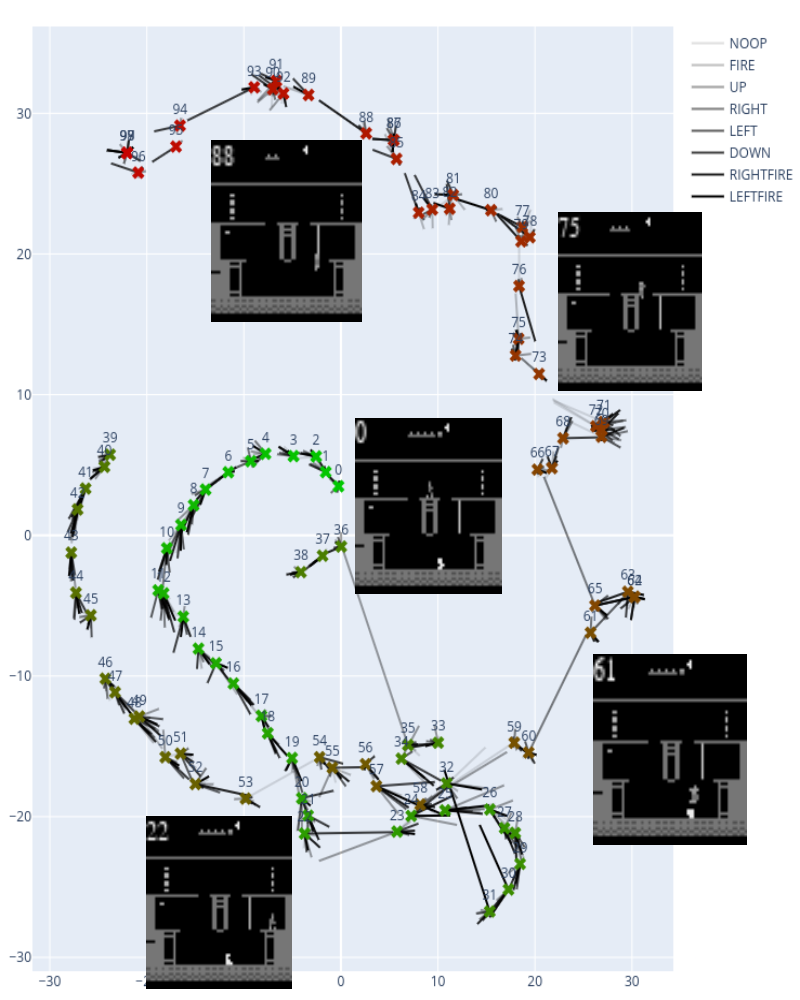}
    \caption{5 dimensional abstract representations visualized with t-SNE.}
    \label{fig:montezuma_abstr_tsne}
    \end{subfigure}
    \ \ 
    \begin{subfigure}[h]{0.45\linewidth}
    \includegraphics[width=\linewidth]{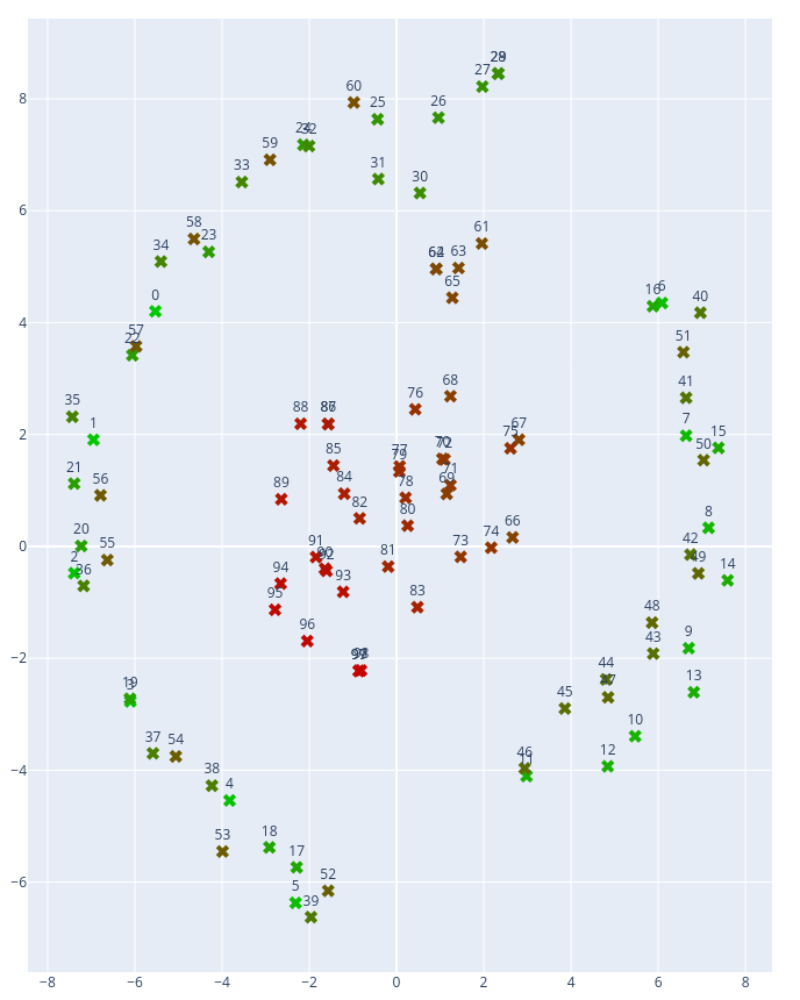}
    \caption{Original observations (each shaped 4 x 64 x 64) visualized with t-SNE.}
    \label{fig:montezuma_obs_tsne}
    \end{subfigure}
    \caption{a) Visualization for 100 observations (4 frames per observation) of Montezuma's Revenge game play. Representation learnt was $n_{\mathcal{X}} = 5$ and visualized with t-SNE \citep{laurens2008tsne} in 2 dimensions. Labels on top-left of game frames correspond to labels of states in lower-dimensional space. Transitions are shown by shaded lines. b) Original resized game frames visualized using t-SNE with the same parameters.}
\end{figure}
We show preliminary results for learning abstract representations for a more complex task, \textit{Montezuma's Revenge}. Comparing the two figures above, we observe how temporally closer states are closer together in lower-dimensional learnt representational space as compared to pixel space. Transitions are not shown for raw observations.
\newpage
\section{Labyrinth state count visualization}
\label{app:uniform-exploration}
\begin{figure}[ht!]
    \centering
    \captionsetup{width=\linewidth}

    \includegraphics[width=0.5\linewidth]{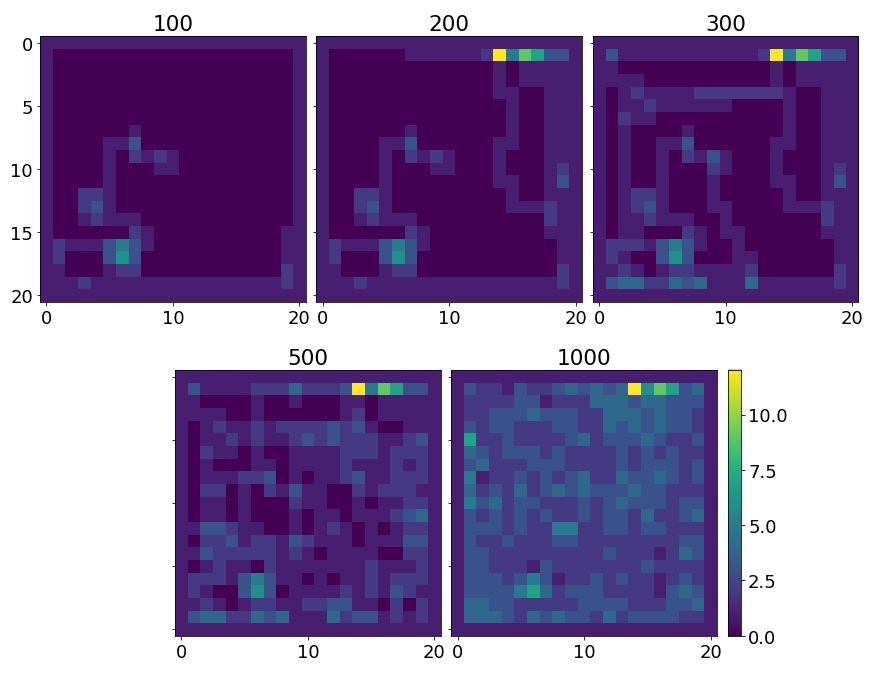}
    \caption{An example of the state counts of our agent in the open labyrinth with $d = 5$ step planning. Titles of each subplot denotes the number of steps taken. The brightness of the points are proportional to the state visitation count. The bright spots that begins after 200 counts is the agent requiring a few trials for learning the dynamics of labyrinth walls.}
\end{figure}

\newpage

\section{Gridworld visualizations}
\label{appendix:gridworld}

\begin{figure}[ht!]
    \centering
    \captionsetup{width=\linewidth}

    \begin{subfigure}[h]{0.25\linewidth}
    \includegraphics[width=\linewidth]{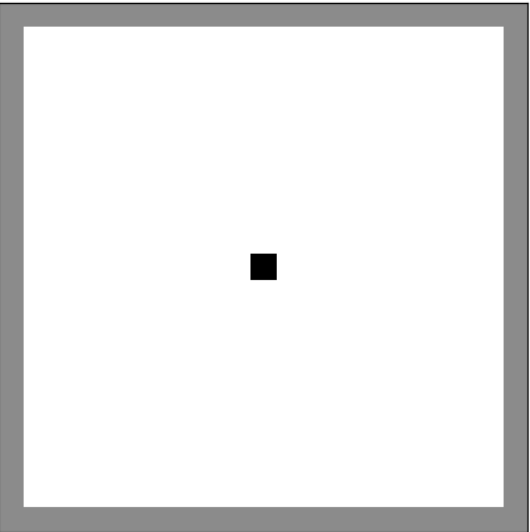}
    \caption{}
    \label{fig:empty_gridworld}
    \end{subfigure}
    \ \
    \begin{subfigure}[h]{0.25\linewidth}
    \includegraphics[width=\linewidth]{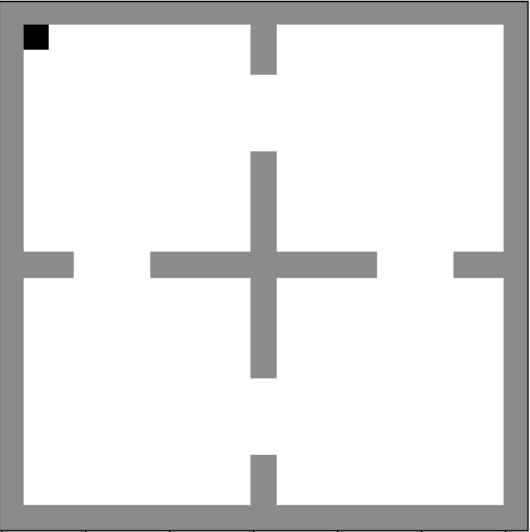}
    \caption{}
    \label{fig:walled_gridworld}
    \end{subfigure}
    \caption{\textit{Left}: Open labyrinth - A $21 \times 21$ empty labyrinth environment. \textit{Right}: 4-room labyrinth - A $21 \times 21$ 4-room labyrinth environment inspired by the 4-room domain in \cite{sutton1999mdp}.}
    \label{fig:gridworld_examples}
\end{figure}

\section{Experimental setup and hyperparameters}
\label{app:hyperparams}
For all of our experiments, we use a batch size of 64 and take 64 random steps transitions before beginning training. We also use the same discount factor for all experiments ($\gamma = 0.8$) and the same freeze interval for target parameters 1000. The reason behind our low discount factor is due to the high density of non-stationary intrinsic rewards in our state. We also use a replay buffer size corresponding to the maximum number of steps in each environment for all experiments. For all model-based abstract representation training, the following hyperparameters were all kept constant: minimum distance between consecutive states $\omega = 0.5$, slack ratio $\delta = 12$ and transition model dropout of $0.1$. For all experiments run with our novelty metric, we use $k = 5$ for our k-NN calculations. For all experiments that allows for forward planning and not explicitly mention depth $d$, we set planning depth $d = 5$. For abstract representation size ($n_{\mathcal{X}}$), we use a dimensionality of 2 for both labyrinth exploration tasks, a dimensionality of 4 for Acrobot, and finally a dimensionality of 3 for the multi-step maze.

\subsection{Neural Network Architectures}
For reference, 'Dense' implies a full-connected layer. 'Conv2D' refers to a 2D convolutional layer with stride 1. 'MaxPooling2D' refers to a max pooling operation. All networks were trained with the RMSProp optimizer. Throughout all experiments, we use the following neural network architectures:
\subsubsection{Encoder}
For all our non-control task inputs, we flatten our input and use the following feed-forward neural network architecture for $\hat{e}$:
\begin{itemize}
    \item Dense(200, activation='tanh')
    \item Dense(100, activation='tanh')
    \item Dense(50, activation='tanh')
    \item Dense(10, activation='tanh')
    \item Dense(abstract representation dimension).
\end{itemize}

For our control task, we use a convolution-based encoder:
\begin{itemize}
    \item Conv2D(channels=8, kernel=(3,3), activation='tanh')
    \item Conv2D(channels=16, kernel=(3,3), activation='tanh')
    \item MaxPool2D(pool size=(4,4))
    \item Conv2D(channels=32, kernel=(3,3), activation='tanh')
    \item MaxPool2D(pool size=(3,3))
    \item Dense(abstract state representation dimension).
\end{itemize}

\subsubsection{Transition model}
The input to our transition model is a concatenation of an abstract representation and an action. We use the following architecture
\begin{itemize}
    \item Dense(10, activation='tanh', dropout=0.1)
    \item Dense(30, activation='tanh', dropout=0.1)
    \item Dense(30, activation='tanh', dropout=0.1)
    \item Dense(10, activation='tanh', dropout=0.1)
    \item Dense(abstract representation dimension)
\end{itemize}
and add the output of this to the input abstract representation.

\subsubsection{Reward and discount factor models}
For both reward and discount factor estimators, we use the following architecture:
\begin{itemize}
    \item Dense(10, activation='tanh')
    \item Dense(50, activation='tanh')
    \item Dense(20, activation='tanh')
    \item Dense(1).
\end{itemize}

\subsubsection{Q function approximator}
We use two different architecture based on the type of input. If we use the concatenation of abstract representation and action, we use the following architecture:
\begin{itemize}
    \item Dense(20, activation='relu')
    \item Dense(50, activation='relu')
    \item Dense(20, activation='relu')
    \item Dense($n_{actions}$)
\end{itemize}
For the pixel frame inputs for our control environments, we use:
\begin{itemize}
    \item Conv2D(channels=8, kernel=(3,3), activation='tanh')
    \item Conv2D(channels=16, kernel=(3,3), activation='tanh')
    \item MaxPool2D(pool size=(4,4))
    \item Conv2D(channels=32, kernel=(3,3), activation='tanh')
    \item MaxPool2D(pool size=(3,3))
    \item Dense($n_{actions}$).
\end{itemize}

Finally, for our (purely model free) gridworld environments we use:
\begin{itemize}
    \item Dense(500, activation='tanh')
    \item Dense(200, activation='tanh')
    \item Dense(50, activation='tanh')
    \item Dense(10, activation='tanh')
    \item Dense($n_{actions}$)
\end{itemize}

As for our Bootstrap DQN implementation, we use the same architecture as above, except we replace the final Dense layer with $10$ separate heads (each a Dense layer with $n_{actions}$ nodes).

\subsection{Labyrinth environments}
Both environments used the same hyperparameters except for two: we add an $\epsilon$-greedy ($\epsilon = 0.2$) policy for the 4-room maze, and increased $n_{freq}$ from 1 to 3 in the 4-room case due to unnecessary over-training. We have the following hyperparameters for our two labyrinth environments:
\begin{itemize}
    \item $n_{iters} = 30000$ 
    \item $\alpha = 0.00025$
\end{itemize}

\subsection{Control environment}

In our Acrobot environment, the input to our agent is 4 stacked consecutive pixel frames, where we reduce each frame down to a $32 \times 32$ pixel frame. Our abstract representation dimension is $4$. We use a learning rate of $\alpha = 0.00025$ for all experiments. We train for $n_{iters} = 50000$ for all experiments with the exception of RND and transition loss - this discrepancy is due to time constraints for the latter two experiments which used $n_{iters} = 3000$, as both these experiments used prohibitively more time to run due to the increased number of steps used to reach the goal state of the environment. 

\subsection{Multi-step maze environment}
In our multistep maze environment, the input to our agent is a single $15 \times 15$ frame of an overview of the environment. Our abstract representation dimension is $3$. We use an $\epsilon$-greedy ($\epsilon = 0.1$) policy for this environment. We use $\alpha = 0.00025, n_{iters} = 30000$ for our model-free algorithms and $\alpha = 0.000025, n_{iters} = 50000$ for experiments that include a model-based component. 
Each episode is at most $4000$ environment steps.

\section{Potential improvements and future work}
\label{appendix:improvements}
\subsection{Incorporating agent history for better generalization}
As mentioned in Section~\ref{sec:4room}, generalization across states while only conditioning on \textit{primary features} ($X, A$ in our case) restricts the generalization ability of our agent. An interesting potential for future work would be to somehow incorporate this information into the learnt representation (potentially by using the same IB method, but using a full history of states as the conditioning variable).

\subsection{Increasing efficiency of learning the abstract state representations}
Currently, learning our low-dimensional state representation takes many iterations per timestep, and is also sensitive to hyperparameter tuning. Our method requires an accurate state representation and dynamics model according to our losses for our method to be effective - the sample-efficiency from our model-learning methods comes at a cost of more time and compute. This is due to the careful balance our model needs to maintain between its losses for good representations. Another interesting potential for future work would be to find ways to incorporate our model-learning losses using less time and computation.

\subsection{Extension to stochastic environments}
One avenue of future work would be to extend this work for stochastic environments. While there has been recent work on using expectation models for planning \citep{wan2020planning} that we could use to extend our algorithm, this still comes with its own limitations. 

\subsection{Empirical validation for representation size}
Another avenue of investigation is to find a more principled approach to finding the right representation size for a given environment. While we currently simply pick the lowest representation size from prior knowledge about an environment, it may be worthwhile to somehow allow the algorithm to decide this.

\end{document}